\documentclass{ieeetj}

\usepackage{cite}
\usepackage{amsmath,amssymb,amsfonts}
\usepackage{amsthm}

\newtheorem{proposition}{Proposition}
\newtheorem{corollary}{Corollary}

\usepackage{graphicx}
\usepackage{textcomp}
\usepackage{xcolor}
\usepackage{booktabs}
\usepackage{multirow}
\usepackage{algorithm}
\usepackage{algorithmic}
\usepackage{hyperref}
\hypersetup{
    colorlinks=true,
    citecolor=blue,
    linkcolor=blue,
    urlcolor=blue,
    pdftitle={PreGS: A Parameter-Transfer-Based Multi-Expert Graph Neural Network for Node Classification},
    pdfauthor={Zhicong Cai, Yinglong Zhang, Xiaoying Hong, Xuewen Xia, Xing Xu}
}

\definecolor{BestRed}{RGB}{220,0,0}
\definecolor{SecondBlue}{RGB}{0,70,180}
\definecolor{ThirdGreen}{RGB}{0,120,0}

\providecommand{\first}[1]{}
\providecommand{\second}[1]{}
\providecommand{\third}[1]{}

\renewcommand{\first}[1]{\textcolor{BestRed}{\textbf{#1\ensuremath{_{\mathrm{(1)}}}}}}
\renewcommand{\second}[1]{\textcolor{SecondBlue}{\textbf{#1\ensuremath{_{\mathrm{(2)}}}}}}
\renewcommand{\third}[1]{\textcolor{ThirdGreen}{\textbf{#1\ensuremath{_{\mathrm{(3)}}}}}}

\def\BibTeX{{\rm B\kern-.05em{\sc i\kern-.025em b}\kern-.08em
    T\kern-.1667em\lower.7ex\hbox{E}\kern-.125emX}}

\AtBeginDocument{
\definecolor{tmlcncolor}{cmyk}{0.93,0.59,0.15,0.02}
\definecolor{NavyBlue}{RGB}{0,86,125}
}

\def\OJlogo{}
\def\seclogo{}

\begin{document}
%\receiveddate{XX Month, XXXX}
%\reviseddate{XX Month, XXXX}
%\accepteddate{XX Month, XXXX}
%\publisheddate{XX Month, XXXX}
%\currentdate{XX Month, XXXX}
%\doiinfo{XXXX.2022.1234567}

\markboth
{PreGS: A Parameter-Transfer-Based Multi-Expert Graph Neural Network for Node Classification}
{Zhicong Cai et al.}

\title{PreGS: A Parameter-Transfer-Based Multi-Expert Graph Neural Network for Node Classification}

\author{Zhicong Cai, Yinglong Zhang, Xiaoying Hong, Xuewen Xia, and Xing Xu}
\affil{College of Physics and Information Engineering, Minnan Normal University, Zhangzhou 363000, China}
\corresp{Corresponding author: Yinglong Zhang (e-mail: zhang\_yinglong@126.com).}

\begin{abstract}
Graph neural networks have achieved strong performance in node classification by aggregating information from graph neighborhoods. However, a single aggregation mechanism may be insufficient to capture diverse structural patterns across graph datasets. Moreover, independently training multiple structural branches can introduce substantial overhead without necessarily producing stable node representations. To address these issues, this paper proposes PreGS, a parameter-transfer-based multi-expert graph neural network framework. PreGS first pretrains a multi-head graph attention network (GAT) and transfers the linear transformation weights of its first-layer attention heads to multiple GraphSAGE experts. The transferred experts are frozen and used as complementary structural branches. The fused raw node features, GAT head representations, and GraphSAGE expert representations are fed into a multilayer perceptron (MLP), whose output is further fused with the pretrained GAT logits. Based on PreGS, we further develop PreGSv2, which introduces source-level weighting and a structural gating mechanism for adaptive multi-source feature integration. Experiments on eight public graph datasets show that PreGS and PreGSv2 achieve competitive performance against representative graph neural network baselines. 
Ablation, parameter-transfer, sensitivity, aggregator, visualization, and training-time analyses further validate the effectiveness and stability of the proposed framework. The code and datasets are available at \url{https://github.com/LH-Czc/PreGS}.

\end{abstract}



\maketitle

\section{INTRODUCTION}
\IEEEPARstart{N}{ode} classification is a fundamental task in graph machine learning and has been widely applied to citation analysis, social networks, recommendation systems, and knowledge discovery. Graph neural networks (GNNs) have become a dominant paradigm for this task by learning node representations through neighborhood aggregation and message passing \cite{gilmer2017mpnn}. Representative models such as the graph convolutional network
(GCN)~\cite{kipf2017gcn}, GraphSAGE~\cite{hamilton2017graphsage},
and the graph attention network (GAT)~\cite{velickovic2018gat}
capture graph structural information through normalized propagation,
fixed neighborhood aggregation, and attention-based neighborhood
weighting, respectively.

Recent studies have further improved GNNs from multiple perspectives, including attention mechanisms, simplified propagation, expressive aggregation functions, multi-scale representation learning, and multi-channel fusion \cite{brody2022gatv2,e2025tangnn,wu2019sgc,xu2019gin,xu2018jknet,corso2020pna,wang2020amgcn}. Meanwhile, graph Transformers and other global modeling methods have attracted increasing attention because of their ability to capture long-range dependencies beyond local message passing \cite{deng2024polynormer,fu2024vcrgraphormer,zhou2026vecformer}. However, recent re-evaluations suggest that classical GNNs remain highly competitive when properly tuned and structurally enhanced. In several node-level and graph-level settings, they can even match or outperform more complex graph Transformer models while retaining better computational efficiency \cite{luo2024classic,luo2025classic}. These observations indicate that the representational potential of classical message-passing architectures has not yet been fully exploited.

Despite these advances, existing methods still face two limitations. First, most GNNs rely on a fixed neighborhood modeling bias within a single propagation branch. For instance, the mean aggregator in GraphSAGE is simple and efficient, but it cannot explicitly distinguish the relative importance of neighboring nodes. GAT introduces attention-based weights to model neighbor importance, but its propagation process is still constrained by a single attention-driven aggregation form. As a result, a single GNN branch may not fully capture the complementary advantages of different neighborhood aggregation mechanisms.

Second, although ensemble learning and mixture-of-experts (MoE) models can increase model capacity, general MoE methods usually depend on expert selection and sparse gating mechanisms \cite{shazeer2017moe,fedus2022switch}. Graph-based MoE methods have also been studied for node classification, decoupled message passing, weak-and-strong expert collaboration, and out-of-distribution graph learning \cite{shi2024moenc,chen2025decoupledmoe,zeng2024weakstrong,sun2026diverse}. Nevertheless, these methods often require independently trained experts or additional routing modules, which increases computational cost and makes it difficult to ensure explicit parameter correspondence and representation-space consistency among experts.

Related to this direction, graph-neural-network-to-multilayer-perceptron
(GNN-to-MLP) distillation methods attempt to transfer structural knowledge
from GNNs to lightweight MLPs or expert models, thereby reducing training
or inference costs
\cite{zhang2022glnn,tian2024dgkd,rumiantsev2024gkdmoe,eskandari2026infgrand}. Recent adaptive hierarchical distillation methods further align GNN and MLP representations across different layers and dimensions to reduce information loss caused by representation mismatch \cite{zhang2025ahkd}. However, these approaches mainly focus on teacher-student knowledge transfer and make limited use of the internal parameter structures of trained GNNs to construct graph experts with aligned representations.

Motivated by these observations, we propose PreGS, a parameter-transfer-based multi-expert graph neural network framework. The key idea is to revisit GNNs from a unified neighborhood aggregation perspective and interpret the difference between GAT and GraphSAGE as different strategies for constructing neighborhood aggregation weights. Since the two models share compatible linear feature transformation structures, the parameters of different attention heads in the first layer of a pretrained multi-head GAT can be transferred to multiple GraphSAGE experts. In this way, PreGS constructs a structurally homologous multi-expert system with explicit parameter correspondence and aligned representation spaces.

During training, the pretrained GAT and the transferred GraphSAGE experts are frozen, while only the fusion module and the MLP classifier are optimized. This design decouples graph structural feature extraction from task-specific adaptation, thereby reducing optimization complexity and improving training stability. The base model, PreGS, adopts a grouped weighted fusion strategy to separately integrate GAT multi-head features and GraphSAGE expert features. The enhanced model, PreGSv2, further introduces source-level weighting and a structural gating mechanism to achieve finer-grained adaptive modeling of multi-source features.

The main contributions of this paper are summarized as follows:
\begin{itemize}
    \item We present a unified neighborhood aggregation perspective and analyze the transferable relationship between GAT and GraphSAGE in terms of their linear feature transformation structures.
    \item We propose a GraphSAGE multi-expert construction mechanism based on pretrained GAT parameter transfer, forming an expert system with structural homology and aligned representation spaces.
    \item We introduce a decoupled training paradigm with frozen experts, where only the fusion module and the MLP classifier are trained, thereby reducing optimization complexity and improving training stability.
    \item We design two fusion models, PreGS and PreGSv2, and conduct a
comprehensive evaluation through node classification, ablation,
parameter-transfer, sensitivity, aggregator, visualization, and
training-time experiments.
\end{itemize}
\section{RELATED WORK}

Existing GNNs usually differ in how neighborhood information is weighted, aggregated, and fused. 
GCN~\cite{kipf2017gcn} adopts normalized graph convolution, GraphSAGE~\cite{hamilton2017graphsage} uses predefined neighborhood aggregators, and GAT~\cite{velickovic2018gat} assigns adaptive attention weights to neighboring nodes. 
Later variants such as GATv2~\cite{brody2022gatv2},
TANGNN~\cite{e2025tangnn}, SGC~\cite{wu2019sgc},
GIN~\cite{xu2019gin}, JK-Net~\cite{xu2018jknet}, and
PNA~\cite{corso2020pna} further improve attention expressiveness,
scalable attention-based neighborhood selection, propagation efficiency,
aggregation capacity, cross-layer representation fusion, and
multi-aggregator modeling. 
However, these methods are usually designed as independent architectures, and the transferable relationship between the internal parameters of different GNN branches has received less attention. 
This motivates our investigation of whether the first-layer attention-head parameters of a pretrained GAT can be reused to construct structurally aligned GraphSAGE experts.

Beyond conventional local message passing, recent studies have explored global interaction and multi-branch modeling to enhance graph representation learning. 
Graph Transformers introduce global node interactions to capture long-range dependencies~\cite{deng2024polynormer,fu2024vcrgraphormer,zhou2026vecformer}, while multi-channel, multi-scale, multi-aggregator, and graph mixture-of-experts methods improve representation learning through complementary branches, aggregation functions, or expert routing mechanisms~\cite{corso2020pna,wang2020amgcn,shi2024moenc,chen2025decoupledmoe,zeng2024weakstrong,sun2026diverse}. 
Although these methods demonstrate the value of combining multiple structural views, they mainly rely on additional architectures or routing modules. 
In contrast, PreGS constructs aligned GraphSAGE experts by explicitly transferring parameters from pretrained GAT heads, which provides a direct parameter-level connection between different GNN branches.

Another related line of work is decoupled graph learning and GNN-to-MLP distillation, which separate structural representation learning from downstream prediction through structural constraints, knowledge transfer, or hybrid expert modeling~\cite{zeng2024weakstrong,zhang2022glnn,tian2024dgkd,rumiantsev2024gkdmoe,zhang2025ahkd}. 
These methods mainly transfer knowledge through predictions, intermediate representations, or structural regularization. 
PreGS follows the spirit of decoupled learning, but focuses on a different transfer target: it reuses the internal feature-transformation parameters of pretrained GAT heads to construct structurally aligned GraphSAGE experts, and then performs lightweight task adaptation through trainable fusion modules.

\section{PRELIMINARIES}

To describe the proposed framework, this section introduces the graph notation and briefly reviews the basic components used in this work, including GAT, GraphSAGE, MLP, and the node classification loss.

\subsection{Notation}

Let $G=(V,E)$ denote an undirected graph, where $V=\{v_1,v_2,\ldots,v_N\}$ is the node set, $N=|V|$ is the number of nodes, and $E\subseteq V\times V$ is the edge set. The node feature matrix is denoted by $\mathbf{X}\in\mathbb{R}^{N\times d}$, where $d$ is the feature dimension. For node $v_i$, its original feature vector is denoted by $\mathbf{x}_i\in\mathbb{R}^{1\times d}$. The label of node $v_i$ is denoted by $y_i$, and the label set of all nodes is denoted by $\mathbf{Y}=\{y_i\}_{i=1}^{N}$, where $y_i\in\{1,2,\ldots,C\}$ and $C$ is the number of classes.

For node $v_i$, its one-hop neighborhood is denoted by 
$\mathcal{N}(i)=\{v_j \mid (v_i,v_j)\in E\}$. 
When the aggregation process includes the target node itself, the self-loop neighborhood is written as 
$\widetilde{\mathcal{N}}(i)=\mathcal{N}(i)\cup\{v_i\}$.

In this paper, $\sigma(\cdot)$ denotes a nonlinear activation function, $\|$ denotes feature concatenation, $\operatorname{softmax}(\cdot)$ denotes a normalization function, and $\operatorname{AGG}(\cdot)$ denotes a replaceable neighborhood aggregation operator. Common aggregation operators include mean, sum, and max, corresponding to neighborhood averaging, neighborhood summation, and neighborhood maximum aggregation, respectively.

For consistency, the hidden representation of node $v_i$ is denoted by $\mathbf{h}_i$, and a linear transformation parameter is denoted by $\mathbf{W}$. If a model contains $K$ attention heads or $K$ expert branches, the parameter and output representation of the $k$th branch are denoted by $\mathbf{W}^{(k)}$ and $\mathbf{h}_i^{(k)}$, respectively, where $k=1,2,\ldots,K$.

Unless otherwise specified, all node feature vectors in this paper are represented as row vectors. Accordingly, linear transformations, attention computation, and feature concatenation are written in row-vector form.

\subsection{Graph Attention Network}

Graph Attention Network (GAT) assigns adaptive weights to neighboring nodes through an attention mechanism, thereby modeling their different contributions to the target node representation \cite{velickovic2018gat}. For node $v_i$ and its neighbor $v_j\in\widetilde{\mathcal{N}}(i)$, let $\mathbf{h}_i$ and $\mathbf{h}_j$ denote their input representations in the current GAT layer. In the first GAT layer, $\mathbf{h}_i=\mathbf{x}_i$ and $\mathbf{h}_j=\mathbf{x}_j$.

With attention coefficients $\alpha_{ij}$ normalized over $\widetilde{\mathcal{N}}(i)$, the GAT update can be written as
\begin{equation}
\mathbf{h}'_i
=
\sigma
\left(
\sum_{v_j\in\widetilde{\mathcal{N}}(i)}
\alpha_{ij}\mathbf{h}_j\mathbf{W}
\right).
\label{eq:gat_update}
\end{equation}
For a multi-head GAT with $K$ attention heads, the output of the $k$th head is
\begin{equation}
\mathbf{h}^{(k)}_i
=
\sigma
\left(
\sum_{v_j\in\widetilde{\mathcal{N}}(i)}
\alpha^{(k)}_{ij}\mathbf{h}_j\mathbf{W}^{(k)}
\right),
\label{eq:gat_multi_head}
\end{equation}
where $k=1,2,\ldots,K$.

In this paper, GAT serves as the pretrained source model. The independent first-layer head parameters $\{\mathbf{W}^{(k)}\}_{k=1}^{K}$ provide the parameter sources for constructing the subsequent GraphSAGE expert branches.

\subsection{GraphSAGE}

GraphSAGE is a neighborhood-aggregation-based graph neural network for learning node representations in an inductive manner \cite{hamilton2017graphsage}. Unlike GAT, which learns adaptive attention weights for neighboring nodes, GraphSAGE usually adopts a predefined aggregation function to summarize neighborhood features.

When the mean aggregator is used and the target node itself is included in the aggregation set, the GraphSAGE update can be written as
\begin{equation}
\mathbf{h}'_i
=
\sigma
\left(
\frac{1}{|\widetilde{\mathcal{N}}(i)|}
\sum_{v_j\in\widetilde{\mathcal{N}}(i)}
\mathbf{h}_j\mathbf{W}
\right).
\label{eq:graphsage_mean}
\end{equation}
This formulation shows that the mean aggregation in GraphSAGE can be interpreted as a neighborhood weighted-sum process with fixed uniform weights. It therefore shares a common aggregation form with the attention-based weighted aggregation in GAT, providing a structural basis for transferring parameters from GAT attention heads to GraphSAGE expert branches.

\subsection{Multilayer Perceptron}

The Multilayer Perceptron (MLP) is used as the classifier after feature fusion. 
Unlike GNN layers, it does not explicitly use graph topology. 
Given the fused representation $\mathbf{f}_i$ of node $v_i$, the two-layer MLP classifier is defined as
\begin{equation}
\mathbf{z}^{\mathrm{MLP}}_i
=
\operatorname{MLP}(\mathbf{f}_i)
=
\sigma(\mathbf{f}_i\mathbf{W}_1+\mathbf{b}_1)\mathbf{W}_2+\mathbf{b}_2,
\label{eq:mlp}
\end{equation}
where $\mathbf{z}^{\mathrm{MLP}}_i$ denotes the MLP logits of node $v_i$. 
In the proposed framework, the MLP performs task-specific prediction based on the fused representations.

\subsection{Node Classification Loss}

For node classification, this paper uses the cross-entropy loss as the training objective. Let $y_i$ denote the ground-truth label of node $v_i$, and let $\mathbf{z}_i$ denote the final logits produced by the model. For a training node set $\mathcal{V}_{\mathrm{tr}}$, the cross-entropy loss is defined as
\begin{equation}
\begin{aligned}
\mathcal{L}_{\mathrm{CE}}
=&
-\frac{1}{|\mathcal{V}_{\mathrm{tr}}|}
\sum_{v_i\in\mathcal{V}_{\mathrm{tr}}}
\sum_{c=1}^{C}
\mathbf{1}(y_i=c) \\
&\times
\log
\left(
\frac{\exp(z_{ic})}
{\sum_{r=1}^{C}\exp(z_{ir})}
\right),
\end{aligned}
\label{eq:cross_entropy}
\end{equation}
where $C$ is the number of classes, $\mathbf{1}(\cdot)$ is the indicator function, and $z_{ic}$ denotes the final logit of node $v_i$ for class $c$. The predicted label is obtained by $\hat{y}_i=\arg\max_c z_{ic}$.

In PreGS and PreGSv2, this loss is used to optimize the fusion module and the MLP classifier, while the pretrained GAT and the transferred GraphSAGE experts are kept fixed.

\section{PROPOSED METHOD}

This section presents the proposed PreGS family framework. We first analyze the structural relationship between GAT and GraphSAGE from a unified neighborhood aggregation perspective. Then, we describe the three-stage procedure of PreGS, including GAT pretraining, GraphSAGE expert construction, and decoupled fusion training. Finally, we introduce the base model PreGS, the enhanced model PreGSv2, and the corresponding algorithmic procedure.

\subsection{Theoretical Basis}
\label{sec:theoretical_basis}

To explain the rationale behind the proposed framework, we start from a unified view of neighborhood aggregation. Although existing graph neural networks adopt different mechanisms for neighborhood modeling, their common objective is to update the representation of a target node by aggregating information from its neighbors. For example, GAT learns adaptive neighborhood weights through attention, while GraphSAGE summarizes neighborhood features using a predefined aggregation function. Despite their architectural differences, both models can be interpreted as neighborhood feature aggregation processes.

\subsubsection{Unified Neighborhood Aggregation Perspective}

To characterize the relationship between GAT and GraphSAGE, we formulate their aggregation processes from a unified neighborhood aggregation perspective.

\begin{proposition}
\label{prop:unified_aggregation}
For graph neural networks that adopt linear feature transformation and neighborhood information aggregation, the node representation update process of weighted-sum-based aggregators can be written in the following unified form:
\begin{equation}
\mathbf{h}'_i
=
\sigma
\left(
\sum_{v_j\in\widetilde{\mathcal{N}}(i)}
\beta_{ij}\mathbf{h}_j\mathbf{W}
\right),
\label{eq:unified_aggregation}
\end{equation}
where $\mathbf{W}$ is the feature transformation matrix, and $\beta_{ij}$ denotes the aggregation weight assigned by node $v_i$ to node $v_j$. Different aggregation mechanisms mainly differ in the construction of $\beta_{ij}$.
\end{proposition}

\begin{IEEEproof}
For GAT, by setting $\beta_{ij}=\alpha_{ij}$, the GAT update rule in \eqref{eq:gat_update} can be written in the unified form of \eqref{eq:unified_aggregation}. In this case, $\alpha_{ij}$ is normalized over the neighborhood of node $v_i$.

For GraphSAGE with the mean and sum aggregators, by setting $\beta_{ij}=1/|\widetilde{\mathcal{N}}(i)|$ and $\beta_{ij}=1$, respectively, the corresponding aggregation rules can also be written in the unified form of \eqref{eq:unified_aggregation}.

Therefore, the above aggregation rules can all be expressed under the unified weighted-sum aggregation form.
\end{IEEEproof}

This unified neighborhood aggregation view indicates that GAT, GraphSAGE-mean, and GraphSAGE-sum share compatible linear feature transformation structures and differ mainly in their neighborhood summarization strategies. Therefore, the parameters learned by a pretrained GAT have the potential to be transferred to GraphSAGE expert structures, and the multi-head attention architecture of GAT provides a natural basis for constructing multiple experts. In the proposed framework, different GraphSAGE experts can further adopt different aggregation operators to increase representation diversity, as described in the experimental configuration.

\subsection{PreGS Family Framework}

The central problem addressed in this paper is how to use the naturally parallel multi-head structure of GAT to construct multiple GraphSAGE expert branches with explicit parameter correspondence and aligned representation spaces. To this end, we propose the PreGS family, a parameter-transfer-based multi-expert graph neural network framework, including the base model PreGS and the enhanced model PreGSv2. Instead of training independent experts or using complex routing mechanisms commonly adopted in MoE-based graph models~\cite{shi2024moenc,chen2025decoupledmoe,zeng2024weakstrong}, PreGS transfers the parameters of the first-layer attention heads of a pretrained GAT to multiple GraphSAGE experts, making the experts structurally homologous and explicitly parameter-related.

In the subsequent training stage, both the pretrained GAT and the transferred GraphSAGE experts are frozen. Only the fusion module and the MLP classifier are optimized. This decoupled design preserves the representation ability of the pretrained model while reducing the optimization complexity of later training.

\subsection{Overall Procedure}

The proposed method consists of three stages: GAT pretraining, multi-expert construction with parameter transfer, and decoupled fusion training.

\noindent\textbf{1) GAT pretraining.}
In the first stage, a multi-head GAT is pretrained on the target graph to obtain stable attention-head parameters and graph structural representations. Let $\mathbf{h}^{(k)}_i$ denote the output representation of node $v_i$ from the $k$th attention head in the first GAT layer, where $k=1,2,\ldots,K$. Meanwhile, the final output logits of the pretrained GAT are retained and denoted by $\mathbf{z}^{\mathrm{GAT}}_i$. These logits contain high-level structural information obtained after multi-layer attention propagation and are used as an important supplement in the final prediction fusion. After pretraining, the GAT parameters are fixed and no longer updated during the subsequent fusion training stage.

\noindent\textbf{2) Multi-expert construction and parameter transfer.}
In the second stage, the linear transformation parameters of the $k$th attention head in the first GAT layer are transferred to the $k$th GraphSAGE expert:
\begin{equation}
\mathbf{W}^{(k)}_{\mathrm{GS}}
\leftarrow
\mathbf{W}^{(k)}_{\mathrm{GAT}},
\label{eq:parameter_transfer}
\end{equation}
where $k=1,2,\ldots,K$.
This operation constructs $K$ single-layer GraphSAGE expert branches. Each expert corresponds to one GAT attention head, so the experts are not randomly initialized or independently trained from scratch. Instead, they are derived from the multi-head structure of the pretrained GAT. The output representation of node $v_i$ produced by the $k$th single-layer GraphSAGE expert is denoted by $\mathbf{e}^{(k)}_i$. After expert construction, all GraphSAGE expert parameters are frozen.

\noindent\textbf{3) Decoupled fusion training.}
In the third stage, the pretrained GAT and the transferred GraphSAGE experts act as fixed feature extractors and do not participate in parameter updates. The model only trains the fusion module and the MLP classifier. Based on this design, this paper develops two fusion strategies. PreGS performs grouped weighted fusion over GAT multi-head features and GraphSAGE expert features, while PreGSv2 further introduces source-level weighting and structural gating for more adaptive multi-source feature modeling.

Because the parameters of the pretrained GAT and the GraphSAGE experts are frozen, the training process only needs to optimize the branch-fusion weights, logit-fusion weights, and MLP classifier parameters, together with the source-level weighting and gating parameters in PreGSv2. This design transforms the highly coupled end-to-end optimization problem in conventional GNNs into a lightweight task-adaptation problem, thereby reducing training complexity and improving training stability.

Figure~\ref{fig:detailed_framework} further presents the detailed architecture of the PreGS family. It shows how the pretrained GAT branch, the transferred GraphSAGE expert branch, and the raw feature branch are integrated before the final fusion-and-prediction stage. The detailed formulations of PreGS and PreGSv2 are given in the following subsections.

\begin{figure*}[!t]
\centering
\includegraphics[width=\textwidth]{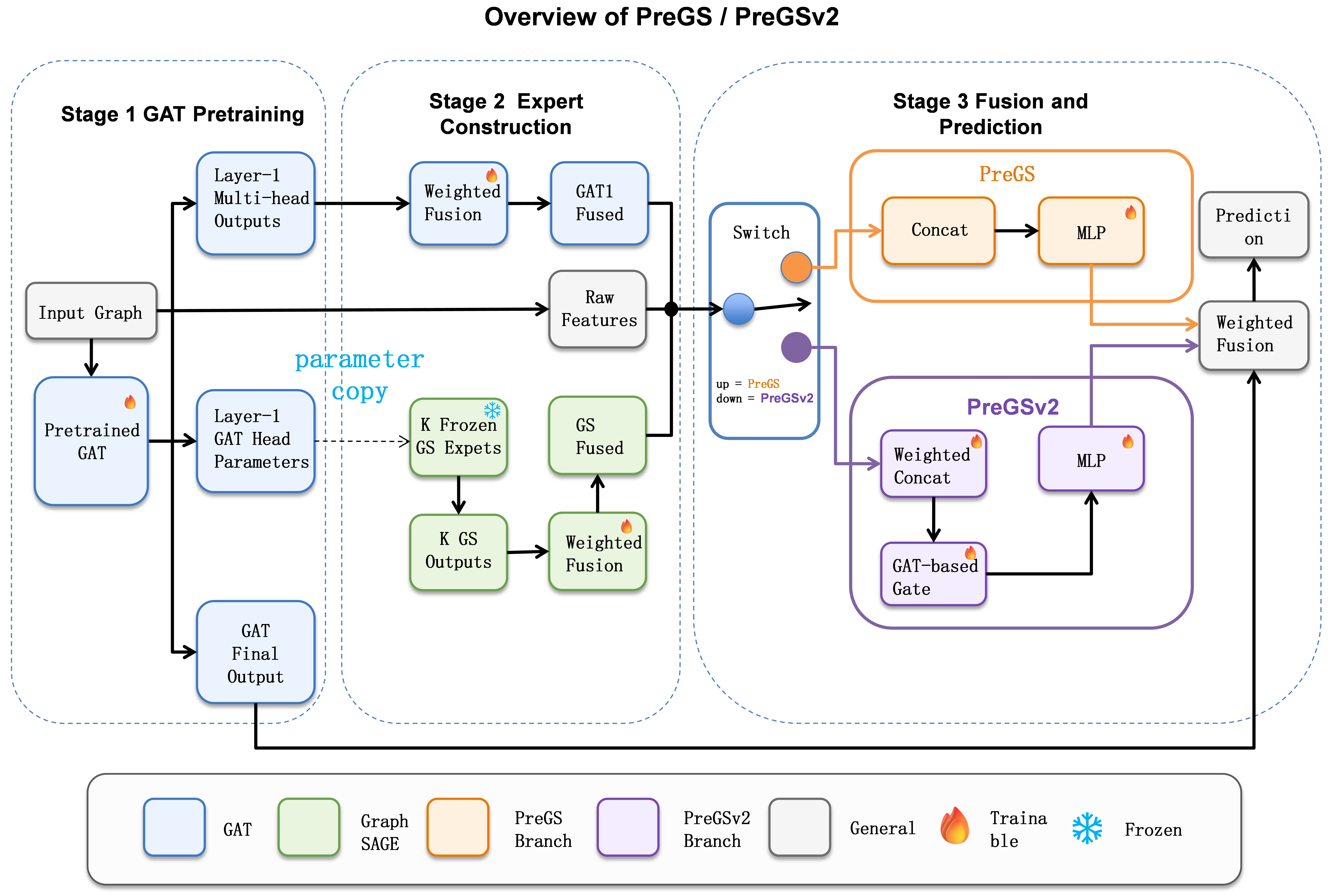}
\caption{Detailed architecture of the proposed PreGS and PreGSv2 framework. The figure illustrates GAT pretraining, GraphSAGE expert construction through parameter transfer, and the fusion-and-prediction stage used by the two model variants.}
\label{fig:detailed_framework}
\end{figure*}

\subsection{PreGS: Grouped Weighted Fusion}

For each node $v_i$, PreGS uses three types of feature sources: the raw node feature $\mathbf{x}_i$, the first-layer GAT head representations $\{\mathbf{h}^{(k)}_i\}_{k=1}^{K}$, and the transferred GraphSAGE expert representations $\{\mathbf{e}^{(k)}_i\}_{k=1}^{K}$.

To model the importance of different branches, learnable weights are introduced for the GAT heads and the GraphSAGE experts, respectively. For the GAT branch, the normalized weight of the $k$th head is defined as
\begin{equation}
\omega^{\mathrm{GAT}}_k
=
\frac{\exp(\widetilde{\omega}^{\mathrm{GAT}}_k)}
{\sum_{t=1}^{K}\exp(\widetilde{\omega}^{\mathrm{GAT}}_t)},
\label{eq:pregs_gat_weight}
\end{equation}
where $\{\widetilde{\omega}^{\mathrm{GAT}}_k\}_{k=1}^{K}$ are learnable parameters. For the GraphSAGE expert branch, the normalized expert weight is defined as
\begin{equation}
\omega^{\mathrm{GS}}_k
=
\frac{\exp(\widetilde{\omega}^{\mathrm{GS}}_k)}
{\sum_{t=1}^{K}\exp(\widetilde{\omega}^{\mathrm{GS}}_t)},
\label{eq:pregs_gs_weight}
\end{equation}
where $\{\widetilde{\omega}^{\mathrm{GS}}_k\}_{k=1}^{K}$ are also learnable parameters.

The two groups of features are then fused separately by weighted summation:
\begin{equation}
\mathbf{f}^{\mathrm{GAT}}_i
=
\sum_{k=1}^{K}
\omega^{\mathrm{GAT}}_k \mathbf{h}^{(k)}_i,
\label{eq:pregs_gat_fusion}
\end{equation}
\begin{equation}
\mathbf{f}^{\mathrm{GS}}_i
=
\sum_{k=1}^{K}
\omega^{\mathrm{GS}}_k \mathbf{e}^{(k)}_i.
\label{eq:pregs_gs_fusion}
\end{equation}

After that, the raw node feature and the two fused representations are concatenated to obtain the final fusion representation:
\begin{equation}
\mathbf{f}_i
=
\mathbf{x}_i
\|
\mathbf{f}^{\mathrm{GAT}}_i
\|
\mathbf{f}^{\mathrm{GS}}_i.
\label{eq:pregs_concat}
\end{equation}
The fused representation is then fed into an MLP classifier:
\begin{equation}
\mathbf{z}^{\mathrm{MLP}}_i
=
\operatorname{MLP}(\mathbf{f}_i).
\label{eq:pregs_mlp_output}
\end{equation}

Since the final output of the pretrained GAT contains high-level structural information, PreGS further combines the MLP logits with the final GAT logits through learnable logit-level fusion:
\begin{equation}
\mathbf{z}_i
=
\frac{\exp(\lambda_1)}
{\exp(\lambda_1)+\exp(\lambda_2)}
\mathbf{z}^{\mathrm{MLP}}_i
+
\frac{\exp(\lambda_2)}
{\exp(\lambda_1)+\exp(\lambda_2)}
\mathbf{z}^{\mathrm{GAT}}_i,
\label{eq:pregs_logit_fusion}
\end{equation}
where $\lambda_1$ and $\lambda_2$ are learnable fusion parameters, and $\mathbf{z}_i$ denotes the final logits of node $v_i$.

During training, the pretrained GAT and the transferred GraphSAGE experts are frozen. Only the fusion weights and the MLP classifier parameters are optimized. The training objective follows the cross-entropy loss defined in \eqref{eq:cross_entropy}.

\subsection{PreGSv2: Source-Weighted Gated Fusion}

Although PreGS can effectively combine the structural information from the pretrained GAT and the GraphSAGE experts, its fusion weights are shared across all nodes after training. Therefore, it cannot dynamically adjust the importance of different feature sources according to the structural characteristics of individual nodes. To improve the adaptive modeling ability of the framework, we further propose PreGSv2, which introduces source-level weighting and a structural gating mechanism on the basis of PreGS.

\subsubsection{Source-Level Weighted Fusion}

PreGSv2 introduces learnable source-level weights for the raw feature, the fused GAT representation, and the fused GraphSAGE representation. The normalized source weight is defined as
\begin{equation}
\rho_s
=
\frac{\exp(\eta_s)}
{\sum_{t=1}^{3}\exp(\eta_t)},
\label{eq:source_weight}
\end{equation}
where $\{\eta_s\}_{s=1}^{3}$ are learnable parameters. The source-weighted representation is then constructed by weighted concatenation:
\begin{equation}
\mathbf{f}^{\mathrm{src}}_i
=
\rho_1\mathbf{x}_i
\|
\rho_2\mathbf{f}^{\mathrm{GAT}}_i
\|
\rho_3\mathbf{f}^{\mathrm{GS}}_i,
\label{eq:source_weighted_concat}
\end{equation}
where $\mathbf{f}^{\mathrm{GAT}}_i$ and $\mathbf{f}^{\mathrm{GS}}_i$ are obtained from \eqref{eq:pregs_gat_fusion} and \eqref{eq:pregs_gs_fusion}, respectively.

\subsubsection{Structural Gating Mechanism}

To further enhance node-level adaptive modeling, PreGSv2 introduces a structural gating mechanism. The gate vector is generated from the fused GAT representation:
\begin{equation}
\mathbf{g}_i
=
\sigma
\left(
\mathbf{f}^{\mathrm{GAT}}_i\mathbf{W}_g+\mathbf{b}_g
\right),
\label{eq:gate_vector}
\end{equation}
where $\mathbf{W}_g$ and $\mathbf{b}_g$ are learnable parameters. The gate vector is then used to modulate the source-weighted representation dimension by dimension:
\begin{equation}
\mathbf{f}^{\mathrm{gate}}_i
=
\mathbf{g}_i
\odot
\mathbf{f}^{\mathrm{src}}_i,
\label{eq:gated_feature}
\end{equation}
where $\odot$ denotes the Hadamard product.

The gated representation is fed into the MLP classifier:
\begin{equation}
\mathbf{z}^{\mathrm{MLP}}_i
=
\operatorname{MLP}
\left(
\mathbf{f}^{\mathrm{gate}}_i
\right).
\label{eq:pregsv2_mlp_output}
\end{equation}
As in PreGS, the MLP logits are further combined with the final GAT logits:
\begin{equation}
\mathbf{z}_i
=
\frac{\exp(\lambda_1)}
{\exp(\lambda_1)+\exp(\lambda_2)}
\mathbf{z}^{\mathrm{MLP}}_i
+
\frac{\exp(\lambda_2)}
{\exp(\lambda_1)+\exp(\lambda_2)}
\mathbf{z}^{\mathrm{GAT}}_i.
\label{eq:pregsv2_logit_fusion}
\end{equation}

Compared with PreGS, PreGSv2 keeps the parameter-transfer mechanism and the multi-expert structure unchanged, but introduces source-level weighting and structural gating to dynamically adjust different feature sources. This improves the ability of the model to represent complex graph structures. PreGSv2 is trained with the same cross-entropy loss defined in \eqref{eq:cross_entropy}.

\subsection{Further Theoretical Analysis}

The previous subsections construct a parameter-transfer-based multi-expert graph neural network framework from the unified neighborhood aggregation view. This subsection further analyzes the framework from three aspects: the feasibility of parameter transfer, representation-space consistency after transfer, and the multi-expert interpretation of the transferred GraphSAGE branches.

\begin{proposition}
\label{prop:parameter_transfer}
Assume that a pretrained GAT attention head and a GraphSAGE expert have the same input and output feature dimensions. Then, the feature transformation parameter $\mathbf{W}$ learned by the GAT attention head can be transferred to the GraphSAGE expert as its initialization parameter, regardless of the specific neighborhood aggregation operator used by the expert.
\end{proposition}

\begin{IEEEproof}
For a fixed input and output feature space, the matrix $\mathbf{W}$ represents the feature mapping between node representations. The neighborhood aggregation operator determines how transformed neighborhood features are summarized, but it does not change the dimensional compatibility of $\mathbf{W}$. Therefore, when the input and output dimensions are consistent, the feature transformation parameter learned by a GAT attention head can be transferred to a GraphSAGE expert.
\end{IEEEproof}

Proposition~\ref{prop:parameter_transfer} shows that the feasibility of parameter transfer comes from the shared feature transformation structure of the two models, rather than from identical neighborhood aggregation mechanisms. Therefore, the pretrained GAT can provide effective initialization for GraphSAGE experts and avoid constructing each expert from random initialization.

\begin{corollary}
\label{cor:representation_consistency}
Under the parameter-transfer condition, the GAT attention heads and the GraphSAGE experts share the same linear transformation parameters. Therefore, their output representations lie in aligned feature transformation spaces and can be directly fused at the feature level.
\end{corollary}

\begin{IEEEproof}
Let the transformation matrix of the $k$th GAT attention head be $\mathbf{W}^{(k)}_{\mathrm{GAT}}$. According to the parameter-transfer operation, the corresponding GraphSAGE expert is initialized with the same transformation matrix, namely $\mathbf{W}^{(k)}_{\mathrm{GS}}=\mathbf{W}^{(k)}_{\mathrm{GAT}}$. Thus, both branches map the input features into the same transformed feature dimension. Although their neighborhood aggregation weights are different, their feature mapping spaces remain aligned. Therefore, the GAT head representation $\mathbf{h}^{(k)}_i$ and the GraphSAGE expert representation $\mathbf{e}^{(k)}_i$ have representation-space consistency.
\end{IEEEproof}

Representation-space consistency ensures that $\mathbf{f}^{\mathrm{GAT}}_i$ and $\mathbf{f}^{\mathrm{GS}}_i$ can be weighted and concatenated directly without introducing an additional projection layer. This provides the theoretical basis for the fusion strategy in PreGS.

\begin{proposition}
\label{prop:multi_expert_construction}
Suppose the first layer of a pretrained GAT contains $K$ attention heads, whose feature transformation matrices are denoted by $\{\mathbf{W}^{(k)}_{\mathrm{GAT}}\}_{k=1}^{K}$. If the parameter of each attention head is transferred to one GraphSAGE expert, the resulting GraphSAGE expert set can be regarded as a structurally homologous reconstruction and extension of the GAT multi-head representation in the GraphSAGE architecture.
\end{proposition}

\begin{IEEEproof}
Each GraphSAGE expert inherits the feature transformation parameter learned by a different GAT attention head. Therefore, different experts preserve the feature projection ability of different attention heads. Meanwhile, because the experts may use different aggregation operators or neighborhood weighting patterns, they produce diverse expert representations. Thus, the transferred GraphSAGE experts form a parameter-related and structurally corresponding expert set derived from the pretrained GAT multi-head structure.
\end{IEEEproof}

Proposition~\ref{prop:multi_expert_construction} indicates that multiple GraphSAGE experts can be interpreted as a reconstruction and extension of the pretrained GAT multi-head knowledge under the GraphSAGE architecture. Therefore, the proposed framework does not simply combine independent models; instead, it builds a multi-expert system with explicit parameter origins and structural correspondence.

In summary, the proposed framework is supported by parameter-transfer feasibility, representation-space consistency, and multi-expert construction. First, the unified neighborhood aggregation view explains why GAT parameters can be transferred to GraphSAGE structures. Second, shared initialization parameters keep different experts in aligned representation spaces. Finally, the structural correspondence between GAT multi-head attention and GraphSAGE experts provides an interpretation for multi-expert construction. These analyses jointly provide the theoretical basis for the proposed multi-expert graph neural network framework.

\subsection{Algorithmic Procedure}

Algorithm~\ref{alg:pregs_family} summarizes the training procedure of PreGS and PreGSv2.

\begin{algorithm}[!t]
\caption{Training Procedure of PreGS and PreGSv2}
\label{alg:pregs_family}
\footnotesize
\begin{algorithmic}[1]
\REQUIRE Graph $G=(V,E)$, features $\mathbf{X}$, labels $\mathbf{Y}$, heads/experts $K$, variant $m$
\ENSURE Node predictions $\{\hat{y}_i\}_{v_i\in V}$

\STATE Train a $K$-head GAT and freeze it after pretraining.
\STATE Store first-layer GAT-head representations $\{\mathbf{h}^{(k)}_i\}_{k=1}^{K}$ and final GAT logits $\mathbf{z}^{\mathrm{GAT}}_i$.

\FOR{$k=1$ to $K$}
    \STATE Build a single-layer $\mathrm{GS}_k$ by parameter transfer in \eqref{eq:parameter_transfer}, assign its aggregator, freeze it, and extract $\mathbf{e}^{(k)}_i$.
\ENDFOR

\REPEAT
    \STATE Compute $\mathbf{f}^{\mathrm{GAT}}_i$ and $\mathbf{f}^{\mathrm{GS}}_i$ by \eqref{eq:pregs_gat_weight}--\eqref{eq:pregs_gs_fusion}.
    \IF{$m=\mathrm{PreGS}$}
        \STATE Construct $\mathbf{f}_i$ by concatenation in \eqref{eq:pregs_concat}.
        \STATE Obtain $\mathbf{z}^{\mathrm{MLP}}_i$ and $\mathbf{z}_i$ by \eqref{eq:pregs_mlp_output} and \eqref{eq:pregs_logit_fusion}.
    \ELSE
        \STATE Construct $\mathbf{f}^{\mathrm{gate}}_i$ by source weighting and gating in \eqref{eq:source_weight}--\eqref{eq:gated_feature}.
        \STATE Obtain $\mathbf{z}^{\mathrm{MLP}}_i$ and $\mathbf{z}_i$ by \eqref{eq:pregsv2_mlp_output} and \eqref{eq:pregsv2_logit_fusion}.
    \ENDIF
    \STATE Update trainable parameters by minimizing \eqref{eq:cross_entropy}.
\UNTIL{early stopping criterion is satisfied}

\STATE Predict labels from the final logits.
\RETURN $\{\hat{y}_i\}_{v_i\in V}$
\end{algorithmic}
\end{algorithm}
\section{EXPERIMENTS}

\subsection{Experimental Setup}

\subsubsection{Datasets}

Experiments are conducted on eight public graph datasets: ACM, AMAC, AMAP, DBLP, EAT, FILM, PubMed, and Texas. These datasets cover different graph domains, including academic networks, citation networks, Amazon co-purchasing networks, air-traffic networks, actor co-occurrence networks, and webpage networks \cite{wang2019han,yang2016revisiting,ribeiro2017struc2vec,pei2020geomgcn}. They vary in graph scale, feature dimensionality, and number of classes, providing a diverse benchmark for node classification. The dataset statistics are summarized in Table~\ref{tab:dataset_statistics}.

\begin{table}[!t]
\caption{Statistics of the Datasets}
\label{tab:dataset_statistics}
\centering
\setlength{\tabcolsep}{7pt}
\renewcommand{\arraystretch}{1.12}
\footnotesize
\begin{tabular}{lccc}
\toprule
Dataset & Nodes ($N$) & Feature Dim. ($d$) & Classes ($C$) \\
\midrule
ACM    & 3025  & 1870 & 3 \\
AMAC   & 2405  & 128  & 4 \\
AMAP   & 1043  & 128  & 6 \\
DBLP   & 4057  & 334  & 4 \\
EAT    & 1575  & 64   & 5 \\
FILM   & 778   & 932  & 5 \\
PubMed & 19717 & 500  & 3 \\
Texas  & 183   & 1703 & 5 \\
\bottomrule
\end{tabular}
\end{table}

\subsubsection{Experimental Configuration}

To ensure fair comparison, all models follow the same data splitting and training protocol, and the GNN baselines are implemented with a two-layer architecture where applicable. For PreGS and PreGSv2, a two-layer GAT with eight attention heads is first pretrained as the source model, and the output dimension of each head is set to 8. The first-layer attention-head parameters are then transferred to eight single-layer GraphSAGE experts. For the default eight-head setting, PreGS and PreGSv2 use the expert aggregation configuration 
$(\mathrm{mean}, \mathrm{mean}, \mathrm{max}, \mathrm{max}, \mathrm{max}, \mathrm{sum}, \mathrm{sum}, \mathrm{sum})$.
This mixed configuration is adopted to introduce complementary neighborhood summarization behaviors among the transferred experts. Specifically, mean and sum aggregators follow the weighted-sum aggregation form discussed in Section~\ref{sec:theoretical_basis}, while the max aggregator is included as an additional summarization operator to enrich expert diversity. 

All models are trained with the Adam optimizer. The learning rate is set to 0.005, the weight decay coefficient is set to $5\times 10^{-4}$, and the dropout rate is set to 0.6. The maximum number of training epochs is 2000, and early stopping with a patience of 100 is used. Unless otherwise specified, each quantitative experiment is repeated 30 times, and the average result is reported.

For data splitting, the training ratios are set to 20\%, 40\%, and 60\%, respectively. The validation ratio is fixed at 10\%, and the remaining nodes are used for testing. All compared models are optimized using the same cross-entropy loss defined in \eqref{eq:cross_entropy}.

\subsubsection{Baselines}

We compare the proposed PreGS and PreGSv2 with representative node classification baselines, including GCN~\cite{kipf2017gcn}, SGC~\cite{wu2019sgc}, GIN~\cite{xu2019gin}, GAT~\cite{velickovic2018gat}, GATv2~\cite{brody2022gatv2}, GraphSAGE~\cite{hamilton2017graphsage}, GraphSAGE++~\cite{e2024graphsagepp}, GNNMoE~\cite{chen2025decoupledmoe}, JK-Net~\cite{xu2018jknet}, and MLP.
These baselines cover normalized propagation, simplified propagation, expressive aggregation, attention-based message passing, fixed neighborhood aggregation, multi-scale aggregation, mixture-of-experts-based adaptive message passing, cross-layer fusion, and feature-only classification.
For consistency, GraphSAGE and GraphSAGE++ are trained without neighbor sampling, using the same input graph for message passing as the other GNN baselines.
All compared models follow the same data splitting, validation, and early-stopping protocol.

\subsection{Accuracy Results}

Table~\ref{tab:node_classification_acc} reports the node classification accuracy under different training ratios. Overall, the proposed models show strong and consistent performance across the eight datasets. Among the 24 dataset--training-ratio settings, either PreGS or PreGSv2 achieves the best result in 20 settings, and at least one of the two models ranks among the top two in 21 settings. Specifically, PreGS obtains 11 best and 7 second-best results, while PreGSv2 obtains 9 best and 6 second-best results. These results indicate that parameter transfer and multi-source feature fusion can provide substantial improvements over a single attention-based branch.

More specifically, PreGS achieves the best performance under all three training ratios on ACM, DBLP, and Texas, and also ranks first on AMAP under the 20\% and 40\% training ratios. PreGSv2 achieves the best results under all three training ratios on AMAC and FILM, ranks first on AMAP under the 60\% training ratio, and obtains the best performance on PubMed under the 20\% and 40\% training ratios. GNNMoE is a strong competing baseline, particularly on ACM, AMAP, and PubMed, and achieves the best result on PubMed under the 60\% training ratio. The proposed models do not dominate EAT, where JK-Net and GIN consistently perform better. Overall, PreGS provides stronger performance on ACM, DBLP, and Texas, whereas PreGSv2 shows clearer advantages on AMAC, FILM, and PubMed, demonstrating that the two fusion strategies are complementary across different graph structures.

\begin{table*}[!t]
\caption{Node Classification Accuracy Under Different Training Ratios (\%)}
\label{tab:node_classification_acc}
\centering
\setlength{\tabcolsep}{2.4pt}
\renewcommand{\arraystretch}{1.1}
\scriptsize
\resizebox{\textwidth}{!}{
\begin{tabular}{ll*{12}{c}}
\toprule
Dataset & Train & GCN & SGC & GIN & GAT & GATv2 & GraphSAGE & GraphSAGE++ & GNNMoE & JK-Net & MLP & PreGS & PreGSv2 \\
\midrule

\multirow{3}{*}{ACM}
& 20\% & 91.23 & 86.29 & 89.16 & 90.92 & 91.18 & 90.80 & 90.65 & \second{92.42} & 91.37 & 88.26 & \first{92.63} & \third{92.35} \\
& 40\% & 91.94 & 87.22 & 90.28 & 91.71 & 91.75 & 91.61 & 91.18 & \second{93.41} & 92.12 & 89.85 & \first{93.55} & \third{93.36} \\
& 60\% & 92.42 & 88.52 & 91.49 & 92.33 & 92.38 & 92.15 & 91.86 & \third{93.83} & 92.50 & 90.37 & \first{93.96} & \second{93.85} \\
\midrule

\multirow{3}{*}{AMAC}
& 20\% & 78.59 & 76.11 & 74.73 & 76.40 & 82.10 & 76.98 & 79.48 & \third{83.79} & 81.75 & 81.92 & \second{86.00} & \first{86.29} \\
& 40\% & 79.11 & 77.03 & 76.09 & 76.91 & 83.23 & 77.29 & 81.58 & \third{85.90} & 82.27 & 83.52 & \second{87.18} & \first{87.31} \\
& 60\% & 78.96 & 77.51 & 77.14 & 77.11 & 83.40 & 77.26 & 82.45 & \third{86.72} & 82.45 & 84.19 & \second{87.44} & \first{87.75} \\
\midrule

\multirow{3}{*}{AMAP}
& 20\% & 93.37 & 92.82 & 67.80 & 93.28 & 93.76 & 92.94 & 93.17 & \third{94.03} & 93.61 & 89.46 & \first{94.87} & \second{94.85} \\
& 40\% & 93.67 & 93.16 & 72.84 & 93.54 & 94.03 & 93.12 & 93.97 & \third{94.81} & 93.96 & 91.01 & \first{95.37} & \second{95.36} \\
& 60\% & 93.66 & 93.24 & 66.92 & 93.83 & 94.26 & 93.28 & 94.31 & \third{95.21} & 94.09 & 91.60 & \second{95.52} & \first{95.55} \\
\midrule

\multirow{3}{*}{DBLP}
& 20\% & 81.93 & 71.38 & 79.72 & 81.61 & 81.65 & 81.38 & 79.89 & 81.40 & \third{82.40} & 79.46 & \first{83.08} & \second{82.66} \\
& 40\% & 83.12 & 73.17 & 81.87 & 82.79 & 82.68 & 82.77 & 81.54 & 83.07 & \third{83.46} & 81.23 & \first{84.04} & \second{83.80} \\
& 60\% & 84.06 & 74.09 & 83.03 & 83.32 & 83.34 & 83.63 & 82.64 & 83.89 & \second{84.24} & 81.63 & \first{84.50} & \third{84.10} \\
\midrule

\multirow{3}{*}{EAT}
& 20\% & 47.09 & 41.15 & \second{51.85} & 34.38 & 34.45 & 36.30 & \third{47.71} & 46.30 & \first{54.53} & 39.31 & 47.00 & 47.52 \\
& 40\% & \third{50.40} & 41.26 & \second{54.63} & 39.24 & 38.29 & 37.91 & 48.18 & 49.05 & \first{55.14} & 43.57 & 49.15 & 49.29 \\
& 60\% & \third{51.07} & 40.99 & \second{54.52} & 38.87 & 41.52 & 38.04 & 44.77 & 50.14 & \first{58.35} & 45.73 & 49.23 & 48.84 \\
\midrule

\multirow{3}{*}{FILM}
& 20\% & 28.03 & 25.97 & 26.89 & 27.89 & 27.98 & 28.80 & 28.48 & 35.26 & 28.16 & \third{35.90} & \second{36.12} & \first{36.56} \\
& 40\% & 28.73 & 26.36 & 28.06 & 28.35 & 28.30 & 29.41 & 29.87 & 35.93 & 29.30 & \third{36.98} & \second{36.99} & \first{37.89} \\
& 60\% & 28.97 & 26.65 & 28.64 & 28.61 & 28.60 & 29.63 & 30.69 & 36.99 & 29.70 & \third{37.66} & \second{37.74} & \first{38.30} \\
\midrule

\multirow{3}{*}{PubMed}
& 20\% & 68.55 & 57.66 & 79.13 & 67.35 & 70.72 & 68.98 & 77.15 & 84.94 & 73.84 & \second{85.99} & \third{85.53} & \first{86.67} \\
& 40\% & 68.66 & 58.01 & 82.13 & 67.34 & 71.14 & 69.11 & 77.92 & \third{86.61} & 74.32 & \second{86.70} & 86.48 & \first{87.12} \\
& 60\% & 68.82 & 58.16 & 83.55 & 67.43 & 71.42 & 69.25 & 78.32 & \first{87.68} & 74.58 & \third{87.21} & 86.91 & \second{87.37} \\
\midrule

\multirow{3}{*}{Texas}
& 20\% & 54.83 & 52.87 & 52.51 & 54.34 & 54.44 & 53.51 & \third{55.37} & 54.39 & 54.83 & \second{59.15} & \first{62.35} & 54.68 \\
& 40\% & 54.71 & 53.01 & 52.21 & 53.80 & 54.09 & 54.28 & \third{59.24} & 55.07 & 56.78 & \second{62.72} & \first{66.96} & 54.09 \\
& 60\% & 56.19 & 54.35 & 53.10 & 54.40 & 54.40 & 52.50 & \third{59.88} & 54.35 & 58.45 & \second{69.94} & \first{75.06} & 54.40 \\
\bottomrule
\end{tabular}
}
\vspace{2pt}

{\footnotesize \first{Best}, \second{second-best}, and \third{third-best} mark the top three results in each row.}
\end{table*}

\subsection{Effectiveness of Parameter Transfer}

To examine whether pretrained GAT parameters can provide useful initialization for GraphSAGE experts, we transfer the first-layer attention-head parameters of a two-layer eight-head GAT to eight single-layer GraphSAGE experts. An MLP classifier is then attached after each transferred expert. The experiment is conducted under the 20\% training ratio, and the results are reported in Table~\ref{tab:weight_copy}.

As shown in Table~\ref{tab:weight_copy}, a single transferred GraphSAGE expert generally performs worse than the complete GAT and GraphSAGE models, but most experts still retain non-trivial classification ability. This indicates that the linear transformation parameters learned by GAT attention heads can be reused by GraphSAGE experts as meaningful feature mappings. The results also show that different aggregators adapt to the transferred parameters differently. For example, on AMAC, mean-based experts are more stable than max-based experts. This supports the use of multiple aggregation operators to provide diverse expert representations.

\begin{table*}[!t]
\caption{Effectiveness of Head-Wise Parameter Transfer Under the 20\% Training Ratio}
\label{tab:weight_copy}
\centering
\setlength{\tabcolsep}{5.2pt}
\renewcommand{\arraystretch}{1.12}
\scriptsize
\resizebox{0.88\textwidth}{!}{
\begin{tabular}{lcccccccccc}
\toprule
Dataset & GAT & GraphSAGE & Mean-1 & Mean-2 & Max-1 & Max-2 & Max-3 & Sum-1 & Sum-2 & Sum-3 \\
\midrule
ACM  & 90.91 & 90.80 & 88.76 & 88.24 & 88.08 & 87.42 & 82.91 & 87.90 & 88.50 & 86.56 \\
AMAC & 76.48 & 76.91 & 70.19 & 66.35 & 47.11 & 43.99 & 40.42 & 54.28 & 67.64 & 64.45 \\
\bottomrule
\end{tabular}
}
\vspace{2pt}

{\footnotesize All values denote node classification accuracy (\%).}
\end{table*}

\subsection{Ablation Study}

To analyze the contribution of each component, ablation studies are conducted for PreGS and PreGSv2 under the 20\% training ratio. The results are shown in Table~\ref{tab:ablation}. The Full Model rows report the accuracy of the complete models, while the remaining rows report the accuracy change relative to the corresponding full model.

For PreGS, using only raw features leads to clear performance drops on most datasets, confirming the importance of graph structural information. Removing the final GAT logits also causes declines on several datasets, indicating that the high-level structural information from the pretrained GAT is useful for final prediction. By contrast, removing the GAT1-fused features or GS-fused features has a more dataset-dependent effect. In particular, removing GS-fused features causes a large drop on EAT, showing that the transferred GraphSAGE experts provide useful complementary neighborhood information on this dataset.

For PreGSv2, removing weighted concatenation or gating causes only small changes on most datasets, but clear drops are observed on PubMed. This suggests that source-level weighting and gating are helpful when different feature sources contribute unevenly. However, the positive changes on Texas also indicate that these adaptive modules are not universally beneficial for every graph. Overall, the ablation results show that the main components have meaningful but dataset-dependent contributions.

\begin{table*}[!t]
\caption{Ablation Study of PreGS and PreGSv2 Under the 20\% Training Ratio}
\label{tab:ablation}
\centering
\setlength{\tabcolsep}{4.2pt}
\renewcommand{\arraystretch}{1.08}
\scriptsize
\resizebox{0.98\textwidth}{!}{
\begin{tabular}{llcccccccc}
\toprule
Model & Variant & ACM & AMAC & AMAP & DBLP & EAT & FILM & PubMed & Texas \\
\midrule
\multirow{6}{*}{PreGS}
& Full Model & \textbf{92.64} & \textbf{85.91} & \textbf{94.90} & \textbf{83.08} & \textbf{46.94} & \textbf{36.06} & \textbf{85.47} & \textbf{62.38} \\
& Raw Features Only & -4.32 & -2.23 & -4.15 & -2.98 & -5.78 & -0.36 & 0.53 & 1.65 \\
& w/o GAT Final & -1.38 & -0.38 & -1.28 & -1.51 & 0.43 & -0.37 & 0.06 & 1.55 \\
& w/o GAT1-Fused & -0.27 & 0.30 & -0.05 & -0.12 & 0.31 & 0.03 & 0.34 & 0.44 \\
& w/o GS-Fused & -0.16 & -0.09 & 0.15 & -0.10 & -6.86 & 0.17 & 0.21 & -0.18 \\
& w/o Raw Features & -1.81 & -4.34 & -1.43 & -1.97 & 2.33 & -7.64 & -9.36 & -7.55 \\
\midrule
\multirow{3}{*}{PreGSv2}
& Full Model & \textbf{92.35} & \textbf{86.21} & \textbf{94.92} & \textbf{82.66} & \textbf{47.46} & \textbf{36.57} & \textbf{86.65} & \textbf{54.68} \\
& w/o Weighted Concat & 0.01 & 0.08 & 0.00 & 0.21 & -0.50 & -0.17 & -1.24 & 1.44 \\
& w/o Gating & 0.11 & -0.04 & -0.02 & 0.20 & 0.16 & -0.01 & -0.57 & 1.21 \\
\bottomrule
\end{tabular}
}
\vspace{2pt}

{\footnotesize Full Model rows report accuracy values. Other rows report accuracy changes relative to the corresponding full model.}
\end{table*}

\subsection{Sensitivity to the Number of GAT Heads}

To examine the sensitivity of the proposed framework to the number of GAT heads, we vary the number of attention heads under the 20\% training ratio and compare GAT, PreGS, and PreGSv2. In this experiment, all transferred GraphSAGE experts use the mean aggregator to isolate the effect of the number of attention heads. The results are shown in Table~\ref{tab:gat_heads}.

Across all selected datasets and head settings, PreGS and PreGSv2 consistently outperform the original GAT. PreGS achieves the best performance on ACM and DBLP, while PreGSv2 performs best on AMAC and FILM. This indicates that the proposed framework can obtain stable gains under different head settings, and that the two variants show different advantages across datasets.

Although the performance of GAT may improve when more heads are used on some datasets, the proposed models maintain clear improvements over GAT for 2, 4, and 8 heads. Therefore, the performance gain does not simply come from changing the number of attention heads, but mainly from parameter transfer, multi-expert aggregation, and multi-source feature fusion.

\begin{table}[!t]
\caption{Sensitivity to the Number of GAT Heads Under the 20\% Training Ratio}
\label{tab:gat_heads}
\centering
\setlength{\tabcolsep}{5.0pt}
\renewcommand{\arraystretch}{1.12}
\scriptsize
\begin{tabular}{llccc}
\toprule
Dataset & Heads & GAT & PreGS & PreGSv2 \\
\midrule

\multirow{3}{*}{ACM}
& 2 & \third{90.91} & \first{92.42} & \second{92.27} \\
& 4 & \third{90.93} & \first{92.39} & \second{92.18} \\
& 8 & \third{90.92} & \first{92.62} & \second{92.20} \\
\midrule

\multirow{3}{*}{AMAC}
& 2 & \third{69.95} & \second{85.59} & \first{85.90} \\
& 4 & \third{75.63} & \second{86.00} & \first{86.10} \\
& 8 & \third{76.20} & \second{85.63} & \first{86.02} \\
\midrule

\multirow{3}{*}{DBLP}
& 2 & \third{80.66} & \first{82.84} & \second{82.42} \\
& 4 & \third{81.04} & \first{82.87} & \second{82.36} \\
& 8 & \third{81.61} & \first{83.04} & \second{82.71} \\
\midrule

\multirow{3}{*}{FILM}
& 2 & \third{27.34} & \second{35.93} & \first{36.84} \\
& 4 & \third{27.53} & \second{35.86} & \first{36.78} \\
& 8 & \third{27.89} & \second{36.10} & \first{36.68} \\
\bottomrule
\end{tabular}
\vspace{2pt}

{\footnotesize \first{Best}, \second{second-best}, and \third{third-best} mark the top three results in each row.}
\end{table}

\subsection{Aggregator Combination Analysis}
\label{sec:aggregator_analysis}

To examine the effect of expert aggregation strategies, we compare
representative single-aggregator and mixed-aggregator configurations
under the 20\% training ratio. Table~\ref{tab:aggregator_analysis}
reports the best result and its corresponding configuration for each
dataset.

\begin{table}[!t]
\caption{Best Aggregator Configurations Under the 20\% Training Ratio}
\label{tab:aggregator_analysis}
\centering
\setlength{\tabcolsep}{3.2pt}
\renewcommand{\arraystretch}{1.10}
\footnotesize
\begin{tabular}{lcccc}
\toprule
& \multicolumn{2}{c}{PreGS}
& \multicolumn{2}{c}{PreGSv2} \\
\cmidrule(lr){2-3}
\cmidrule(lr){4-5}
Dataset & Acc. (\%) & Config. & Acc. (\%) & Config. \\
\midrule
ACM    & 92.65 & Mixed    & 92.40 & Max-Sum  \\
AMAC   & 85.96 & Mean-Sum & 86.23 & All-Sum  \\
AMAP   & 95.06 & All-Max  & 95.09 & All-Max  \\
DBLP   & 83.09 & All-Max  & 82.73 & Mean-Sum \\
EAT    & 47.43 & All-Sum  & 48.22 & Max-Sum  \\
FILM   & 36.13 & All-Mean & 36.68 & All-Mean \\
PubMed & 85.59 & All-Mean & 86.66 & Mean-Sum \\
Texas  & 63.28 & All-Max  & 54.68 & All-Mean \\
\bottomrule
\end{tabular}

\vspace{2pt}
{\scriptsize Mixed denotes
(mean, mean, max, max, max, sum, sum, sum).}
\end{table}

The optimal configuration varies across datasets and model variants,
indicating that no single aggregation strategy is universally optimal.
This result supports the use of heterogeneous experts to capture
complementary neighborhood information.

\subsection{Embedding Visualization Analysis}
\label{sec:tsne_visualization}

To qualitatively compare the learned node representations, we visualize
the embeddings of twelve models on AMAC under the 20\% training ratio
using t-distributed stochastic neighbor embedding (t-SNE). As shown in Figures~\ref{fig:tsne_amac_part1} and
\ref{fig:tsne_amac_part2}, several baselines still exhibit noticeable
class overlap. In comparison, PreGS forms more organized local clusters,
while PreGSv2 shows clearer separation in several regions, consistent
with its higher accuracy on AMAC.

\begin{figure}[!t]
\centering
\includegraphics[width=\columnwidth]{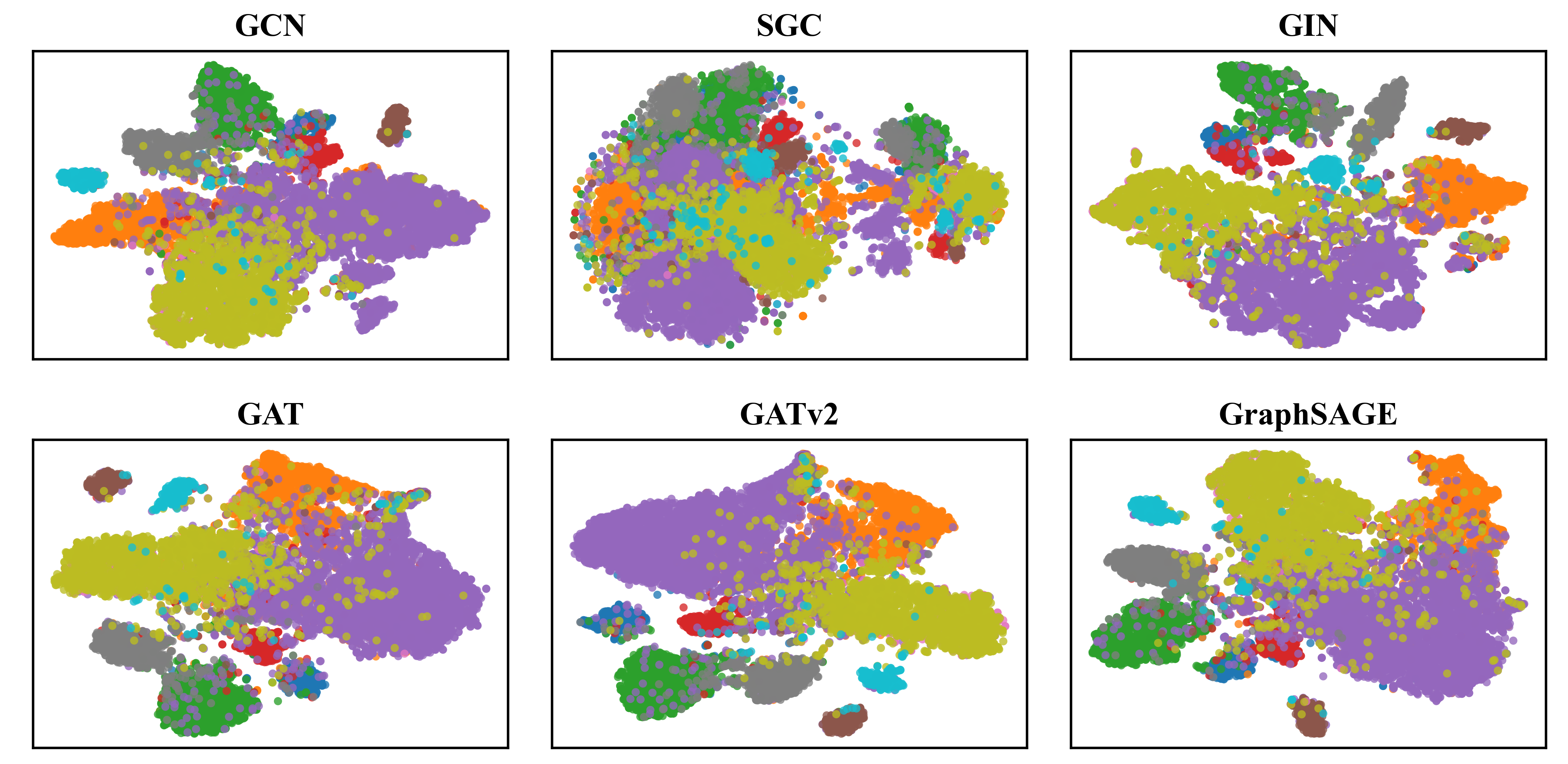}
\caption{t-SNE visualization on AMAC: GCN, SGC, GIN, GAT, GATv2, and GraphSAGE.}
\label{fig:tsne_amac_part1}
\end{figure}

\begin{figure}[!t]
\centering
\includegraphics[width=\columnwidth]{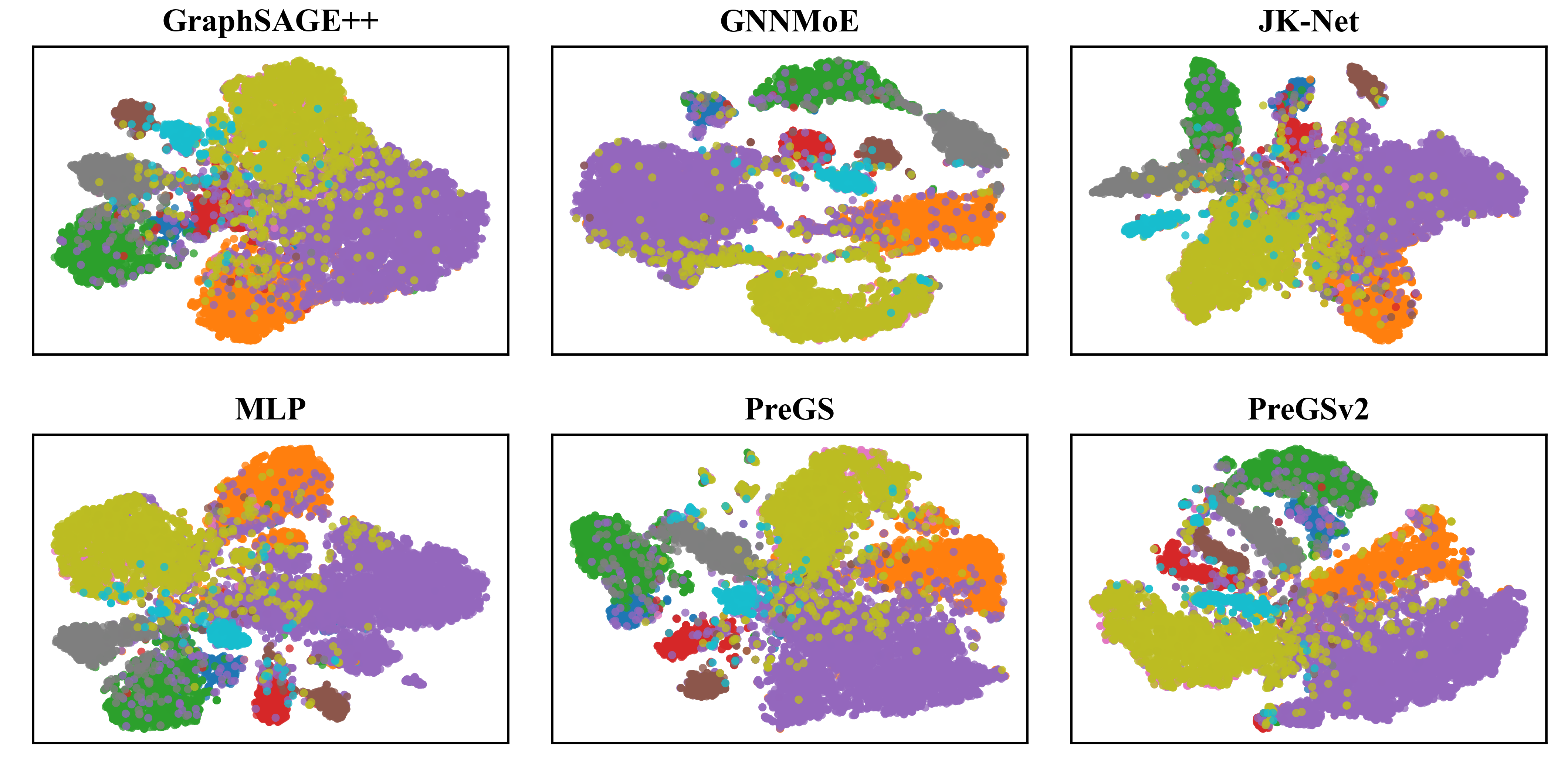}
\caption{t-SNE visualization on AMAC: GraphSAGE++, GNNMoE, JK-Net, MLP, PreGS, and PreGSv2.}
\label{fig:tsne_amac_part2}
\end{figure}

\subsection{Training Time Analysis}

To evaluate the computational overhead of the proposed models, we compare
the end-to-end training convergence time of MLP, GAT, GNNMoE, PreGS, and
PreGSv2 on ACM and AMAC under the 20\% training ratio. For PreGS and
PreGSv2, the reported time includes GAT pretraining and the subsequent
fusion-training stage. The results are reported in
Table~\ref{tab:training_time}.

\begin{table}[!t]
\caption{Training Time Comparison on ACM and AMAC}
\label{tab:training_time}
\centering
\setlength{\tabcolsep}{4.5pt}
\renewcommand{\arraystretch}{1.12}
\footnotesize
\begin{tabular}{lccccc}
\toprule
Dataset & MLP & GAT & GNNMoE & PreGS & PreGSv2 \\
\midrule
ACM  & 0.50 & 2.19 & 5.08  & 3.33  & 3.81  \\
AMAC & 1.08 & 8.49 & 13.58 & 10.76 & 11.03 \\
\bottomrule
\end{tabular}
\vspace{2pt}

{\footnotesize All values are reported in seconds.}
\end{table}

As shown in Table~\ref{tab:training_time}, MLP requires the least training
time because it does not involve graph-based message passing. PreGS and
PreGSv2 incur additional computational cost compared with GAT because of
the parameter-transfer and fusion-training stages. Nevertheless, both
proposed models are consistently more efficient than GNNMoE on the two
datasets. This result indicates that the proposed framework introduces
moderate additional overhead while maintaining a more favorable training
efficiency than the competing graph mixture-of-experts model.

\section{CONCLUSION}

This paper proposed PreGS, a parameter-transfer-based multi-expert graph neural network framework for node classification. PreGS transfers the first-layer attention-head parameters of a pretrained GAT to multiple GraphSAGE experts and freezes both the GAT and the transferred experts during subsequent training. Based on PreGS, PreGSv2 further introduces source-level weighting and structural gating to improve adaptive multi-source feature fusion.

Experiments on eight public graph datasets show that PreGS and PreGSv2 achieve competitive performance against representative GNN baselines. The parameter-transfer, ablation, sensitivity, aggregator, visualization,
and training-time analyses further support the effectiveness of the
proposed framework. Future work will extend the framework to larger-scale graphs and other graph learning tasks.

\end{document}